\documentclass[letterpaper, 10 pt, conference]{ieeeconf}  

\IEEEoverridecommandlockouts                              

\usepackage{nicefrac} 
\usepackage{mathptmx} 
\usepackage{times} 
\usepackage{amsmath} 
\usepackage{amssymb} 

\usepackage{amsthm}

\usepackage{color}
\usepackage{multirow}
\usepackage{algorithm}
\usepackage{algorithmic}
\usepackage{bm}
\usepackage{bbm}
\usepackage{float} 
\usepackage[cal=cm]{mathalfa} 
\usepackage{xcolor}
\usepackage{microtype}
\usepackage{graphicx}
\usepackage{subfigure}
\usepackage{caption}
\usepackage{multicol}
\usepackage{booktabs}
\usepackage{lipsum}

\newtheorem{proposition}{Proposition}
\newtheorem{theorem}{Theorem}
\newtheorem{lemma}{Lemma}

\newtheorem{assumption}{Assumption}
\allowdisplaybreaks[4]

\def\br#1{{#1}}
\def\bl#1{{#1}}
\usepackage[colorlinks=false]{hyperref}
\hypersetup{hidelinks}
\title{\LARGE \bf
Distributed TD(0) with Almost No Communication
}

\author{Rui Liu$^{1}$ and Alex Olshevsky$^{2}$
\thanks{$^{1}$Rui Liu is with the Division of Systems Engineering, Boston University, Boston, MA, USA
        {\tt\small rliu@bu.edu}}%
\thanks{$^{2}$Alex Olshevsky is with the Department of ECE and Division of Systems Engineering, Boston University, Boston, MA, USA
        {\tt\small alexols@u.edu}}%
}

\begin{document}

\maketitle
\thispagestyle{empty}
\pagestyle{empty}

\begin{abstract}

We provide a new non-asymptotic analysis of distributed temporal difference learning with linear function approximation. Our approach relies on ``one-shot averaging,'' where $N$ agents run  identical local copies of the TD(0) method and average the outcomes only once at the very end. We demonstrate a version of the linear time speedup phenomenon, where the convergence time of the distributed process is a factor of $N$ faster than the convergence time of TD(0). This is the first result proving benefits from parallelism for temporal difference methods.

\end{abstract}

\section{Introduction}

Recent years have seen reinforcement learning used in a variety of multi-agent systems. However, a rigorous understanding of how standard methods in reinforcement learning perform in a multi-agent setting with limited communication is only beginning to be available. 

One of the most fundamental problems in reinforcement learning is policy evaluation, and one of the most basic policy evaluation algorithms is temporal difference (TD) learning, originally proposed in \cite{sutton1988learning}. TD learning works by updating a value function from differences in predictions over a succession of steps in the underlying Markov Decision Process (MDP). 

Developments in the field of multi-agent reinforcement learning have led to an increased interest in decentralizing TD methods, which is the subject of this paper. We will consider a simple model where $N$ agents all have access to their own copy of the same MDP. Naturally, the agents can simply ignore each other and run any policy evaluation method without communication. However, this ignores the possibility that agents can benefit from mixing local computations by each agent and inter-agent interactions. Our goal will be to quantify how much TD methods can benefit from this. 



\subsection{Related Literature}

A natural benchmark to compare the performance of distributed TD methods to is the performance of centralized TD methods.  Precise conditions for the asymptotic convergence result { were} first given in \cite{tsitsiklis1997analysis} by viewing TD as a stochastic approximation for solving a Bellman equation. Recently, there has been an increased interest in non-asymptotic convergence results, 
e.g., \cite{dalal2018finite0, bhandari2018finite}. The state of the art results show that, under i.i.d samples, TD algorithm with linear function approximation { converges} with rates of $O(1/\sqrt{T})$ for value function with step-size $1/\sqrt{T}$ and converge as fast as $O(1/t)$ with step-size $O(1/t)$ \cite{bhandari2018finite}.

Prior to this work, there have been several analyses of distributed TD with linear function approximation \cite{sun2019finite, doan2019finite, wang2020decentralized}. However, the model considered by those papers is very different than the model considered here, as those papers considered agents interacting collectively with an environment with a transition function that depends on all the actions taken by the agents. This is a much more difficult setting than what we consider in this paper, where we have $N$ MDPs which are completely decoupled, except insofar as the agents may choose to couple them via an exchange of messages.

Perhaps the most relevant previous work is \cite{shen2020asynchronous} which addresses actor-critic rather than temporal difference methods. It is shown there, up to a certain approximation error, it is possible to obtain a speedup proportional to the number of nodes for a distributed model of actor-critic. { Besides \cite{shen2020asynchronous}, another example of a similar result we are aware of is \cite{khodadadian2022federated}. However, \cite{khodadadian2022federated} came after the present work (note that \cite{khodadadian2022federated} cites the arxiv version \cite{liu2021distributed} of the present paper, which appeared on the arxiv about a year before the arxiv version of \cite{khodadadian2022federated}). The paper \cite{khodadadian2022federated} considers the much more general problem of distributed (or federated) stochastic approximation, which includes temporal difference learning as studied here, alongside many other problems (such as $Q$-learning). A linear speedup is obtained similar to our results here, but it requires 
$N$ averages throughout the course of the algorithm -- in contrast to the single averaging round required in this work.}

\subsection{Our contributions}  We show a version of a ``linear speedup'' phenomenon: under a number of assumptions, we show that the convergence bounds of a distributed algorithm with $N$ agents is a factor of $N$ faster than the corresponding convergence time bounds associated with a centralized version. To our knowledge, this is the first example of this phenomenon being demonstrated in reinforcement learning. 


These results arguably justify the introduction of our model in this paper, which should be contrasted with the much harsher models  considered in the previous multi-agent reinforcement learning literature.  Indeed, the model presented here allows for the possibility of speeding up reinforcement learning by parallelizing computations.

\section{Preliminaries}

We begin by standardizing notation and providing standard background information on Markov Decision Processes and temporal difference methods.

\subsection{Markov Decision Processes}
A discounted reward MDP is described by a 5-tuple $(\mathcal{S},\mathcal{A},\mathcal{P},r,\gamma)$, where $\mathcal{S}=[n]=\{1,2,\cdots,n\}$ is a finite state space, $\mathcal{A}$ is a finite action space, $\mathcal{P}(s'|s,a):\mathcal{S} \times \mathcal{A} \times \mathcal{S} \rightarrow [0,1]$ is transition probability from $s$ to $s'$ determined by $a$, $r(s,a,s'): \mathcal{S} \times \mathcal{A} \times \mathcal{S} \rightarrow \mathbb{R}$ are deterministic rewards and $\gamma \in (0,1)$ is the discount factor. 

Let $\mu$ denote a fixed policy that maps a state $s \in \mathcal{S}$ to a probability distribution $\mu(\cdot|s)$ over the action space $\mathcal{A}$, so that $\sum_{a \in \mathcal{A}} \mu(a|s) =1$. For such a fixed policy $\mu$, define the instantaneous reward vector $R^{\mu} : \mathcal{S} \rightarrow \mathbb{R}$ as $$R^{\mu}(s)=\sum_{s' \in \mathcal{S}} \sum_{a \in \mathcal{A}} \mu(a|s)\mathcal{P}(s'|s,a)r(s,a,s').$$ Fixing the policy $\mu$ induces a probability transition matrix between states: $$P^{\mu}(s,s')=\sum_{a \in \mathcal{A}} \mu(a|s)\mathcal{P}(s'|s,a).$$ We will use $r_{t}=r(s_t,a_t,s_{t+1})$ to denote the instantaneous reward at time $t$, where $s_t$, $a_t$ are the state and action taken at step $t$. The value function of $\mu$, denoted by $V^{\mu}: \mathcal{S} \rightarrow \mathbb{R}$ is defined as 
$V^{\mu}(s)=E_{\mu,s}\left[\sum_{t=0}^{\infty} \gamma^tr_{t}\right]$, where $E_{\mu,s} \left[ \cdot \right]$ indicates that $s$ is the initial state and the actions are chosen according to the policy $\mu$. In the following, we will treat $V^{\mu}$ and $R^{\mu}$ as vectors in $\mathbb{R}^n$ and treat $P^{\mu}$ as a matrix in $\mathbb{R}^{n \times n}$. 

Next, we state { a standard assumption} on the underlying Markov chain.
\begin{assumption} \label{ass:mc}
The Markov chain with transition matrix $P^{\mu}$ is irreducible and aperiodic. 
\end{assumption}
A consequence of Assumption \ref{ass:mc} is that there exists a unique stationary distribution $\pi = (\pi_1, \pi_2, \cdots, \pi_n)$, a row vector whose entries are positive and sum to $1$. This stationary distribution satisfies $\pi^T P^{\mu} =\pi^T$ and $\pi_{s'} = \lim_{t \rightarrow \infty} (P^{\mu})^t(s,s')$ for any two states $s,s' \in \mathcal{S}$. {\em Note that we use $\pi$ to denote the stationary distribution and $\mu$ to denote the policy.}


We next provide definitions of two norms that we will have occasion to use later. For a positive definite matrix $A \in \mathbb{R}^{n \times n}$, we define the inner product $\langle x, y \rangle_{A} = x^T A y$ and the associated norm $\|x\|_A = \sqrt{x^T A x}$ respectively. Since the numbers $\pi_{s}$ are positive for all $s \in \mathcal{S}$, then the diagonal matrix $D = {\rm diag}(\pi_1,\cdots,\pi_n) \in \mathbb{R}^{n \times n}$ is positive definite. Therefore, for any two vectors $V,V' \in \mathbb{R}^n$, we can also define an inner product as $\left \langle V,V' \right \rangle _{D} = V^T D V'=\sum_{s \in \mathcal{S}}\pi_s V(s) V'(s),$and the associated norm as $$ \label{eq:def_D}
    \|V\|_{D}^2=V^T D V = \sum_{s \in \mathcal{S}} \pi_s V(s)^2.
$$

Finally, we introduce the definition of Dirichlet seminorm, following the notation of \cite{ollivier2018approximate}: $$\|V\|_{{\rm Dir}}^2 = \frac{1}{2} \sum_{s,s' \in \mathcal{S} } \pi_s P^{\mu}(s,s') (V(s')-V(s))^2.$$

\subsection{Temporal Difference Learning}

We next introduce the update rule of the classical temporal difference method with linear function approximation $V_{\theta}$, a linear function of $\theta$: 
\begin{equation}\label{eq:V_theta}
 V_{\theta}(s)=\sum_{l=1}^K \theta_l \phi_l(s) \quad \forall s \in \mathcal{S},
\end{equation}
where $\phi_l = (\phi_l(1),\cdots,\phi_l(n))^T \in \mathbb{R}^n$ for $l \in [K]$ are $K$ given feature vectors. Together, all K feature vectors form a $n \times K$ matrix $\Phi = (\phi_1, \cdots,\phi_K)$. For $s \in \mathcal{S}$, let $\phi(s) = (\phi_1(s),\cdots,\phi_K(s))^T \in \mathbb{R}^K$ denote the $s$-th row of matrix $\Phi$, a vector that collects the features of state $s$. Then, Eq. (\ref{eq:V_theta}) can be written in a compact form $V_{\theta}(s) = \theta^T \phi(s)$. For brevity, we will omit the superscript $\mu$ throughout from now on.

The TD(0) method maintains a parameter $\theta(t)$ which is updated at every step to improve the approximation. Supposing that we observe a sequence of states $\{s(t)\}_{t \in \mathbb{N}_0}$, then the classical TD($0)$ algorithm updates as $$\theta ({t+1})= \theta (t) + \alpha_t \delta(t) \phi(s(t)),$$
where $\{\alpha_t \}_{t \in \mathbb{N}_0}$ is the sequence of step-sizes, and letting $s'(t)$ denote the next state after $s(t)$, the quantity $\delta (t)$ is the temporal difference error $$\delta(t) = r (t) + \gamma \theta ^T(t) \phi(s' (t)) - \theta ^T(t) \phi(s (t)).$$



A common assumption on feature vectors in the literature \cite{tsitsiklis1997analysis,bhandari2018finite} is that features are linearly independent and uniformly bounded, which is formally given next.
\begin{assumption}\label{ass:features}
The matrix $\Phi$ has full column rank, i.e., the feature vectors $\{ \phi_1, \ldots, \phi_K\}$ are linearly independent. Additionally, we have that $ \|\phi(s)\|_2^2 \leq 1$ for $s \in \mathcal{S}$.
\end{assumption}

Under Assumption \ref{ass:mc} and \ref{ass:features}, we introduce the steady-state feature covariance matrix $\Phi^T D \Phi$. { Note that,} this is a positive definite matrix as an immediate consequence of Assumptions \ref{ass:mc} and \ref{ass:features}, and we let $\omega > 0$ be a lower bound on its smallest eigenvalue. 


We will use the fact, shown in \cite{tsitsiklis1997analysis}, that under Assumptions \ref{ass:mc}-\ref{ass:features} as well as an additional assumption on the decay of the step-sizes $\alpha_t$, the sequence of iterates $\{\theta_t\}$ generated by TD($0$) learning converges almost surely { to} a vector satisfying a certain projected Bellman equation; we will use $\theta^*$ to refer to this vector.

\subsection{The Distributed Model} 



We consider the scenario where each agent has its own independently evolving copy of the same MDP. More formally, each agent has the same 6-tuple $(\mathcal{S},\mathcal{V},\mathcal{A}, \mathcal{P}, r, \gamma)$; at time $t$, agent $v$ will be in a state $s_v(t)$; it will apply action $a_{v}(t) \in \mathcal{A}$ with probability $\mu( a_v(t)|s_v(t))$; then agent $v$ moves to state $s'_v(t)$ with probability $\mathcal{P}(s'_v(t)| s_v(t), a_v(t))$, {with the transitions of all agents being independent of each other}; finally agent $v$ gets a reward $r_{v}(t) = r(s_v(t),a_v(t),s'_v(t))$. Note that, although the rewards obtained by different agents can be different, the reward function $r(s,a,s')$ is identical across agents. 


Naturally, each agent can easily compute $\theta^*$ by simply ignoring all the other agents and running TD(0) locally. However, this ignores the possibility that agents can benefit from communication with each other. Along these lines, we propose our main method below as Algorithm \ref{algo:our}: each agent runs TD($0$) locally without any communication, and, at the end, the agents simply average the results. 

\begin{algorithm}[ht]
\caption{Parallel TD($0$)} \label{algo:our}
\begin{algorithmic}[1]
\STATE For {$v \in \mathcal{V}$}, initialize $\theta_v(0)$, $s_v(0)$ 
\FOR {$t=0$ to $T-1$} 
 \FOR{$v \in \mathcal{V}$}
 \STATE Observe a tuple $(s_v(t),s_v'(t),r_v(t))$.
 \STATE Compute temporal difference: \begin{equation}\label{eq:delta_lc}
 \delta_v(t) = r_v(t) - \left(\phi(s_v(t))- \gamma \phi(s_v'(t))\right)^T \theta_v(t).
 \end{equation}
 \STATE Execute local TD update: 
 \begin{equation} \label{eq:DTD2}
 {\theta}_v({t+1}) = \theta_v(t) + \alpha_t \delta_v(t)\phi_v(s_v(t)).
 \end{equation}
 \STATE Update the local running average: 
 \[ \hat{\theta}_v(t+1) = \left( 1 - \frac{1}{t+2} \right) \hat{\theta}_v(t) + \frac{1}{t+2} \theta_v(t+1). \]
 \ENDFOR
\ENDFOR
\STATE Return $\hat{\theta}(T) = \frac{1}{N} \sum_{v \in \mathcal{V}} \hat{\theta}_v(T)$ and $\bar{\theta}(T)= \frac{1}{N} \sum_{v \in \mathcal{V}} {\theta}_v(T).$
\end{algorithmic}
\end{algorithm}

\noindent {\em Distributed implementation:} The final averaging step represents the only interactions among the agents. Under the assumption that the nodes are connected to a server, computing the average in step 10 takes a single round of communication with a server. In the more common ``nearest neighbor'' model where the agents are connected over an undirected graph and nodes know the total number of nodes $N$, it is possible to find an $\epsilon$-approximation of the average in $O(N \log (1/\epsilon))$ time using the average consensus algorithm from \cite{olshevsky2017linear}. One could also a finite-time average consensus method, see e.g., \cite{doostmohammadian2020single}. If  knowledge of the number of nodes not available, and the communication graph is further time-varying, it is possible to do the same in $O(N^2 \log (1/\epsilon)$ using the average consensus algorithm from \cite{nedic2009distributed}. Finally, if the underlying graph is directed, one can use the popular push-sum for average consensus method \cite{kempe2003gossip, benezit2010weighted} whose convergence rate is geometric, though the question of whether a version of it can have a polynomial convergence rate in terms of $N$ is open. 

As we will later discuss, it suffices to choose $\epsilon$ in the previous paragraph proportional to a power of $1/T$ (where $T$ is the number of iterations, decided on ahead of time), so that {\em the distributed message complexity of step 10  is  $O(\log T)$ under any of the models discussed. }

\section{Convergence Analyses of Our Method} 

We next describe the main result of this paper, which is an analysis of Algorithm \ref{algo:our} 
under the assumption that the tuples in step 4 are i.i.d. In the literature, the i.i.d model is sometimes referred to as having a ``generator'' for the MDP and is a more restrictive assumption compared to assuming that the state evolves as a Markov process with a fixed starting state. Nevertheless, this is a standard assumption under which many TD and Q-learning methods are analyzed (e.g., \cite{dalal2018finite0,dalal2018finite,doan2019finite,kumar2019sample} among others).

We begin with  some notations. For centralized TD(0), convergence bounds generally scale both with the distance to the initial solution, and with the variance  of the temporal difference error with average reward: $${\sigma}^2 = {\rm E} \left[\left( {r}(s,a,s') - \left(\phi(s)- \gamma \phi(s')\right)^T \theta^* \right)^2\right].$$ Here the expectation is taken with respect to the distribution that generates the state $s$ with probability $\pi_s$, { the actions} $(a_1, \ldots, a_n)$ from the policy, and the next state $s'(t)$ from the transition of the MDP. We will use the same notation in the distributed setting, where this quantity is identical across agents, since the agents are all simulating the same MDP. 

Further, we need a notion of the initial distance to the optimal solution;  for simplicity, we take the maximum over all the agents to define: $$ \hat{R}_0 = \max_{v \in \mathcal{V}} E\left[ \left\|{\theta}_v({0}) - \theta^*\right\|_2^2 \right].$$





The following theorem is our  main result. Note that the equations are color-coded, with the meaning of the colors explained below.  

\begin{theorem} \label{thm:main}
Suppose Assumptions \ref{ass:mc}-\ref{ass:features} hold and suppose that the tuples in step 4 of Algorithm \ref{algo:our} are generated i.i.d. with each $s_v(t)$ sampled from the stationary distribution $\pi$, and $r_v(t)$ being the reward and $s_v'(t)$ being the next state when the action is taken from the policy $\mu$. Then: 

(a) For any constant step-size sequence $\alpha_0 = \cdots = \alpha_T= \alpha \leq (1- \gamma)/8$, we have
\begin{small} 
\begin{align*}
 & E\left[(1-\gamma) \left\|V_{\theta^*} - V_{\hat{\theta}(T)} \right\|_{D}^2 + \gamma \left\|V_{\theta^*} - V_{\hat{\theta}(T)} \right\|_{{\rm Dir}}^2 \right] \\
 & ~~~~~~~~~~  \leq \frac{{\br 1}}{ {\br T }} \left(\frac{{\br 1}}{{\br 2\alpha}} {\br E}\left[ {\br ||\bar{\theta}({0}) - \theta^*||_2^2 } \right] +\frac{{\br 4\hat{R}_0}}{{\br 1-\gamma}} \right)+ \frac{ \alpha {{\bl \sigma}}^2}{{\bl N}} + \frac{{\br 8 \alpha^2 {\sigma}^2}}{{\br 1-\gamma}} .
\end{align*} \end{small} 

(b) For any $T \geq \frac{64}{ (1-\gamma)^2}$ and constant step-size sequence $\alpha_0 = \cdots = \alpha_T =\frac{1}{\sqrt{T}}$, we have
\begin{small} 
\begin{align*}
  & E\left[(1-\gamma) \left\|V_{\theta^*} - V_{\hat{\theta}(T)} \right\|_{D}^2 + \gamma \left\|V_{\theta^*} - V_{\hat{\theta}(T)} \right\|_{{\rm Dir}}^2 \right] \notag \\
 & ~~~~~~~~~~~~~~\leq \frac{1}{2 \sqrt{T} }\left( E\left[ \left\|\bar{\theta}({0}) - \theta^*\right\|_2^2 \right]+\frac{ 2 {{\bl \sigma}}^2}{{\bl N}} \right) +\frac{{\br 1}}{{\br T}} \left(\frac{{\br 4\hat{R}_0+ 8{\sigma}^2} }{{\br 1-\gamma}} \right).
\end{align*}
\end{small} 

(c) For the decaying step-size sequence $\alpha_t = \frac{\alpha}{t+\tau}$ with $\alpha = \frac{2}{(1-\gamma)\omega}$ and $\tau = \frac{16}{(1-\gamma)^2\omega}$. Then, 
\begin{small} 
\begin{align*}
 E \left[ \left\|\bar{\theta}({t+1}) - \theta^*\right\|_2^2 \right] \leq  & \frac{2 \alpha ^2 {{\bl \sigma}}^2/{\bl N}}{t+\tau} + \frac{{\br 8 \alpha^2 \hat{\zeta} }}{{\br (t+\tau)^2}}\\
 & + \frac{{\br (\tau-1)^4E \left[ \left\|\bar{\theta}({0}) - \theta^*\right\|_2^2 \right]} }{{\br (t+\tau)^4}},
\end{align*}
\end{small} where $\hat{\zeta} = \max\left\{ {2 \alpha^2 {\sigma}^2}, \tau \hat{R}_0 \right\}$.
\end{theorem}

The proof of Theorem \ref{thm:main} is given in the { section \ref{sec:proof}}. To parse Theorem \ref{thm:main}, note that all the terms in brown are ``negligible'' in a limiting sense. Indeed, in part (a), the first brown term scales as $O(1/T)$ and consequently goes to zero as $T \rightarrow \infty$ (whereas the remaining terms do not). In parts (b) and (c), the terms in brown go to zero at an asymptotically faster rate compared to the dominant term (i.e., as $1/T$ vs the dominant $1/\sqrt{T}$ term in part(b) and as $1/t^2, 1/t^4$ compared to the dominant $1/t$ in part (c)). Finally, the last term in part (a) scales as $O(\alpha^2)$ and will be negligible compared to the term preceding it, which scales as $O(\alpha)$, when $\alpha$ is small. 

Moreover, among the non-negligible terms, whenever ${\sigma}^2$ appears, it is divided by $N$; this is highlighted in blue. 

{\em To summarize: parts (b) and (c) show that, when the number of iterations is large enough, we can divide the variance term by $N$ as a consequence of the parallelism among $N$ agents. Part (a) shows that, when the number of iterations is large enough and the step-size is small enough, the size of the final error will be divided by $N$. }

Note that, in part (c), the result of this is a factor of $N$ speed up of the entire convergence time (when $T$ is large enough). In part (a), this results in a factor of $N$ shrinking of the asymptotic error (when the step-size $\alpha$ is small enough). In part (b), however, this only shrinks the ``variance term'' by a factor of $N$; the term depending on the initial condition is unaffected. The explanation for this is that in parts (a) and (c), the variance of the temporal difference error dominates the convergence rate, while in part (b) this is not the case.


As far as we are aware, these results constitute the first example where parallelism was shown to be helpful for distributed temporal difference learning. In particular, these results motivate a multi-agent approach even if the underlying MDP is centralized in order to speed-up computation.

\smallskip {\noindent \bf Required accuracy for the averaging step.} For simplicity, we have given Theorem 1 under the assumption
that the final averages $\hat{\theta}(T), \bar{\theta}(T)$  are computed exactly. We now come back to the question of how accurate the final averaging step needs to be to preserve our theoretical guarantees. 
It is immediate that all the quantities we bound in Theorem 1 (i.e., the left-hand sides of all the equations) are Lipschitz
in a neighborhood of $\theta^*$. Thus in Theorem 1(a) we need only a constant error in the averaging step, while in { Theorem 1(b) and 1(c)} we need an error rate proportional to a power or $1/T$. Since all average consensus methods previously discussed compute an $\epsilon$-approximation to average consensus in $O(\log 1/\epsilon)$ steps (treating all other variables as constants), {\em this means that step 10 in our method requires us to run a distributed average consensus method for at most  $O(\log T)$  (treating all variables except $T$ as constants) as previously claimed. }


\section{Proof of our main result}\label{sec:proof}

We now provide the proof of Theorem \ref{thm:main}. Let $\Theta(t)\in \mathbb{R}^{N \times K}$ be a matrix whose rows are $\theta_1^T(t), \cdots, \theta_N^T(t)$. The following proposition follows immediately from the definitions (and recall here our notation of putting a bar to denote the network-wide average). 

\begin{proposition} \label{pro:bar_lc}
Suppose Assumptions \ref{ass:mc}-\ref{ass:features} hold, and suppose that $\{\theta_v(t)\}_{v \in \mathcal{V}} $ are generated by Algorithm \ref{algo:our}. Then,

(a) $\bar{h}(t)$ is a linear function of $\bar{\theta}(t)$ and we can write $\bar{h}(t) = {b} -A \bar{\theta}(t).$ 

(b) The conditional expectation of $\bar{m}(t)$ given $\Theta(t)$ is equal to zero:
\begin{equation}
 E[\bar{m}(t)| \Theta({t}) ] = 0. \label{eq:con_exp_lc}
\end{equation}
\end{proposition}

Our next step is to prove a recurrence relation satisfied by the average of the iterates, stated as the following lemma. 
Recall that $\theta^*$ is the fixed point of TD($0$) on the MDP $(\mathcal{S},\mathcal{V},\mathcal{A}, \mathcal{P}, r, \gamma)$.

\begin{lemma}\label{lem:2}
Suppose Assumptions \ref{ass:mc}-\ref{ass:features} hold. Further suppose that $\{\theta_v\}_{v \in \mathcal{V}}$ are generated by Algorithm \ref{algo:our}. For $t \in \mathbb{N}_0$, we have that 
\begin{small}
    \begin{align}
  & E\left[\left\|\bar{\theta}({t+1}) - \theta^*\right\|_2^2 \right] \leq  E\left[ \left\|\bar{\theta}({t}) - \theta^*\right\|_2^2 \right] \notag \\ & +\alpha_t^2 \left(\frac{2{\sigma}^2}{N} + \frac{8}{N} \sum_{v \in \mathcal{V}} E\left[ \|V_{\theta_v(t)} - V_{\theta^*} \|_{D}^2\right]\right ) \notag \\
 &  - 2 \alpha_t E\left[ (1-\gamma) \left\|V_{\theta^*} - V_{\bar{\theta}(t)} \right\|_{D}^2+ \gamma \left\|V_{\theta^*} - V_{\bar{\theta}(t)} \right\|_{{\rm Dir}}^2 \right]. \label{eq:recc_relation}
\end{align}
\end{small}

\end{lemma}


\begin{proof}[Proof of Lemma \ref{lem:2}] {  We have
$\bar{\theta}(t+1) = \bar{\theta}({t}) + \alpha_t \left[\bar{h}(t) + \bar{m}(t)\right].$
By taking expectations:
\begin{small}
    \begin{align}
 & E\left[\left\|\bar{\theta}({t+1}) - \theta^*\right\|_2^2 \right] 
 =  E\left[ \left\|\bar{\theta}({t}) - \theta^*\right\|_2^2 \right] + \alpha_t^2E\left[ \left\|\bar{h}(t) + \bar{m}(t)\right\|_2^2 \right] 
  \notag \\ & - 2 \alpha_t E\left[ \left(\bar{h}(t)+ \bar{m}(t)\right)^T \left(\theta^* - \bar{\theta}({t})\right)\right] . \label{eq:exp_rec_lc}
\end{align}
\end{small} We first consider the second term on the right hand side of Eq. (\ref{eq:exp_rec_lc}). Following the definition of $\bar{h}(t)$ and $\bar{m}(t)$ and 
plugging in the expression for TD error $\delta_v(t)$ with Eq. (\ref{eq:delta_lc}), we obtain $E\left[ \left\| \bar{h}(t) + \bar{m}(t) \right\|_2^2 \right]
=E \left[\| a^* - b^* \|_2^2\right]$, where \begin{equation*}
    \bm{a}^* =\frac{1}{N} \sum_{v \in \mathcal{V}} \left[ {r_v}(t) - \left(\phi(s_v(t))- \gamma \phi(s_v'(t))\right)^T \theta^*\right]\phi(s_v(t)),
\end{equation*} 
\begin{equation*}
    \bm{b}^* = \frac{1}{N} \sum_{v \in \mathcal{V}} \phi(s_v(t))\left(\phi(s_v(t))- \gamma \phi(s_v'(t))\right)^T( \theta_v(t) -\theta^*).
\end{equation*} Using inequality $\|\bm{a}^*-\bm{b}^*\|^2 \leq 2\|\bm{a}^*\|^2 + 2\|\bm{b}^*\|^2$, we obtain
\begin{align}
 E\left[ \left\|\bar{h}(t) + \bar{m}(t)\right\|_2^2 \right] \leq & 2E\left[\|\bm{a}^*\|^2\right] + 2E\left[\|\bm{b}^*\|^2\right] \notag \\
 \leq &  \frac{2{\sigma}^2}{N} + \frac{8}{N} \sum_{v \in \mathcal{V}} E\left[ \|V_{\theta_v(t)} - V_{\theta^*} \|_{D}^2\right].\label{eq:h+m_lc}
\end{align} We next consider the third term on the right hand side of Eq. (\ref{eq:exp_rec_lc}):
\begin{align}
 & E\left[ \left[\bar{h}(t)+ \bar{m}(t)\right]^T (\theta^* - \bar{\theta}({t}))\right]
 = E\left[\bar{h}^T(t) (\theta^* - \bar{\theta}({t})) \right] \notag \\
 = & E\left[ (1-\gamma)\left\|V_{\theta^*} - V_{\bar{\theta}(t)} \right\|_{D}^2 + \gamma \left\|V_{\theta^*} - V_{\bar{\theta}(t)} \right\|_{{\rm Dir}}^2 \right].\label{eq:inner_product_lc}
\end{align}
Here we use that by Proposition \ref{pro:bar_lc} part (a), we have that 
$\bar{h}(t) = {b} -A \bar{\theta}(t)$. Furthermore, if we let $ \bar{h}(\theta)$ denote the linear function ${b} -A {\theta}$, we have that $\bar{h}(\theta^*)=0$. Now applying Corollary 1 in \cite{liu2021temporal} proves the last equation. }

{  Combining equations (\ref{eq:exp_rec_lc}), (\ref{eq:h+m_lc}), and (\ref{eq:inner_product_lc}), we obtain Eq.(\ref{eq:recc_relation}) } 
\end{proof}

With this lemma in place, we are now ready to provide a proof of Theorem \ref{thm:main}. 

\begin{proof}[Proof of Theorem \ref{thm:main}]
Starting from Lemma \ref{lem:2} and Eq.(\ref{eq:recc_relation}), we first consider the bound for the term $\sum_{t=1}^T \sum_{v=1}^N E\left[ \|V_{\theta_v(t)} - V_{\theta^*} \|_{D}^2\right]$. We can plug in that $N=1$ into Lemma \ref{lem:2} to obtain the next inequality:
\begin{small}
    \begin{align*}
  & E\left[\left\|{\theta}_v({t+1}) - \theta^*\right\|_2^2 \right] \leq  E\left[ \left\|{\theta}_v({t}) - \theta^*\right\|_2^2 \right] \\ & +\alpha_t^2 \left(2{\sigma}^2 + 8 E\left[ \|V_{\theta_v(t)} - V_{\theta^*} \|_{D}^2\right]\right )\\
 &  - 2 \alpha_t E\left[ (1-\gamma) \left\|V_{\theta^*} - V_{{\theta}_v(t)} \right\|_{D}^2+ \gamma \left\|V_{\theta^*} - V_{{\theta}_v(t)} \right\|_{{\rm Dir}}^2 \right].
\end{align*}
\end{small}

If the sequence of step-sizes are non-increasing and satisfies $8\alpha_t^2-2\alpha_t(1-\gamma) \leq - \alpha_t(1-\gamma),$ then we obtain
\begin{align*}
 &  \alpha_t E\left[ (1-\gamma) \left\|V_{\theta^*} - V_{{\theta}_v(t)} \right\|_{D}^2+ 2 \gamma \left\|V_{\theta^*} - V_{{\theta}_v(t)} \right\|_{{\rm Dir}}^2 \right]
 \\ \leq &  E\left[ \left\|{\theta}_v({t}) - \theta^*\right\|_2^2 \right] -E\left[\left\|{\theta}_v({t+1}) - \theta^*\right\|_2^2 \right] + 2 \alpha_t^2 {\sigma}^2.
\end{align*}
Since $ E\big[ 2 \gamma \big\|V_{\theta^*} - V_{{\theta}_v(t)} \big\|_{{\rm Dir}}^2 \big]$ is non-negative, it now follows that
\begin{align*}
 & \alpha_t E\left[ (1-\gamma) \left\|V_{\theta^*} - V_{{\theta}_v(t)} \right\|_{D}^2 \right] \\
 \leq & E\left[ \left\|{\theta}_v({t}) - \theta^*\right\|_2^2 \right] -E\left[\left\|{\theta}_v({t+1}) - \theta^*\right\|_2^2 \right] + 2 \alpha_t^2 {\sigma}^2.
\end{align*}
Multiplying $\alpha_t$ on both sides and summing over $t$, we have
\begin{small}
    \begin{align*}
  & \sum_{t=0}^{T-1} \alpha_t^2 E\left[ (1-\gamma) \left\|V_{\theta^*} - V_{{\theta}_v(t)} \right\|_{D}^2 \right] \\
 = & \alpha_0 E\left[ \left\|{\theta}_v({0}) - \theta^*\right\|_2^2 \right] + \sum_{t=1}^{T-1}(\alpha_{t-1} - \alpha_{t}) E\left[ \left\|{\theta}_v({t}) - \theta^*\right\|_2^2 \right]\\ 
 & -\alpha_{T-1} E\left[\left\|{\theta}_v({T}) - \theta^*\right\|_2^2 \right] + 2 \sum_{t=0}^{T-1} \alpha_t^3 {\sigma}^2\\
 \leq &\alpha_0 E\left[ \left\|{\theta}_v({0}) - \theta^*\right\|_2^2 \right] + 2 \sum_{t=0}^{T-1} \alpha_t^3 {\sigma}^2,
\end{align*}
\end{small}
where the last inequality is because that $\{\alpha_t\}_t$ are non-increasing step-sizes.
Summing over agents $v$, we get \begin{small} 
\begin{align}
 & \sum_{v=1}^N \sum_{t=0}^{T-1} \alpha_t^2 E\left[ (1-\gamma) \left\|V_{\theta^*} - V_{{\theta}_v(t)} \right\|_{D}^2 \right]
 \leq  N \alpha_0 \hat{R}_0 + 2 N \sum_{t=0}^{T-1} \alpha_t^3 {\sigma}^2. \label{eq:sum_vt}
\end{align} \end{small} 
With this equation in place, we now turn to the proof of all the parts of the theorem.

\textbf{Proof of part (a):} We consider the constant step-size sequence $\alpha_0 = \cdots = \alpha_T \leq (1- \gamma)/8$. Then let $\alpha$ denote the constant step-size. Plugging into Eq. (\ref{eq:recc_relation}) and rearranging it, we get
\begin{align*}
 & 2 \alpha E\left[ (1-\gamma) \left\|V_{\theta^*} - V_{\bar{\theta}(t)} \right\|_{D}^2+ \gamma \left\|V_{\theta^*} - V_{\bar{\theta}(t)} \right\|_{{\rm Dir}}^2 \right] \\
 \leq &E\left[ \left\|\bar{\theta}({t}) - \theta^*\right\|_2^2 \right] - E\left[\left\|\bar{\theta}({t+1}) - \theta^*\right\|_2^2 \right]\\
 & +\alpha^2 \left(\frac{2{\sigma}^2}{N} + \frac{8}{N} \sum_{v \in \mathcal{V}} E\left[ \|V_{\theta_v(t)} - V_{\theta^*} \|_{D}^2\right]\right ).
\end{align*}
Summing over $t$ gives

\begin{small}
\begin{align*}
 & 2 \sum_{t=0}^{T-1} \alpha E\left[ (1-\gamma) \left\|V_{\theta^*} - V_{\bar{\theta}(t)} \right\|_{D}^2+ \gamma \left\|V_{\theta^*} - V_{\bar{\theta}(t)} \right\|_{{\rm Dir}}^2 \right]\\
 \leq &E\left[ \left\|\bar{\theta}({0}) - \theta^*\right\|_2^2 \right] - E\left[\left\|\bar{\theta}({T}) - \theta^*\right\|_2^2 \right] \\
 & +  \frac{2 T \alpha^2{\sigma}^2}{N} + \frac{8}{N} \sum_{t=0}^{T-1} \sum_{v \in \mathcal{V}} \alpha^2 E\left[ \|V_{\theta_v(t)} - V_{\theta^*} \|_{D}^2\right]\\
 \leq & E\left[ \left\|\bar{\theta}({0}) - \theta^*\right\|_2^2 \right] + \frac{2 T \alpha^2{\sigma}^2}{N} + \frac{8\alpha}{1-\gamma} \left(\hat{R}_0 + 2 T \alpha^2 {\sigma}^2 \right)
\end{align*}\end{small}where we used Eq. (\ref{eq:sum_vt}). 

Now dividing by $2 \alpha$ on both sides:

\begin{align*}
  & \sum_{t=0}^{T-1} E\left[ (1-\gamma) \left\|V_{\theta^*} - V_{\bar{\theta}(t)} \right\|_{D}^2+ \gamma \left\|V_{\theta^*} - V_{\bar{\theta}(t)} \right\|_{{\rm Dir}}^2 \right] \\
 \leq & \frac{1}{2\alpha}E\left[ \left\|\bar{\theta}({0}) - \theta^*\right\|_2^2 \right] + \frac{ T \alpha {\sigma}^2}{N} + \frac{4}{1-\gamma} \left(\hat{R}_0 + 2 T \alpha^2 {\sigma}^2 \right).
\end{align*}Let $\hat{\theta}(T)=\frac{1}{T} \sum_{t=1}^T \bar{\theta}(t) $. Then, by convexity
\begin{align*}
  & E\left[(1-\gamma) \left\|V_{\theta^*} - V_{\hat{\theta}(T)} \right\|_{D}^2 + \gamma \left\|V_{\theta^*} - V_{\hat{\theta}(T)} \right\|_{{\rm Dir}}^2 \right] \\
 \leq & \frac{1}{T} \sum_{t=1}^T E\left[ (1-\gamma) \left\|V_{\theta^*} - V_{\bar{\theta}(t)} \right\|_{D}^2 + \gamma \left\|V_{\theta^*} - V_{\bar{\theta}(t)} \right\|_{{\rm Dir}}^2 \right]\\
 \leq & \frac{1}{ T } \left(\frac{1}{2\alpha} E\left[ \left\|\bar{\theta}({0}) - \theta^*\right\|_2^2 \right] +\frac{4\hat{R}_0}{1-\gamma} \right)+ \frac{ \alpha {\sigma}^2}{N} + \frac{8 \alpha^2 {\sigma}^2}{1-\gamma},
\end{align*} which is what we wanted to show. 

\textbf{Proof of part (b):} We now consider the step-size $\alpha_0 = \cdots = \alpha_T =\frac{1}{\sqrt{T}}$. When $T \geq \frac{64}{ (1-\gamma)^2}$, it can be observed that $\alpha = \frac{1}{\sqrt{T}} \leq \frac{1-\gamma}{8}$. As a consequence of part (a), it is immediate that, 
\begin{align*}
 & E\left[(1-\gamma) \left\|V_{\theta^*} - V_{\hat{\theta}(T)} \right\|_{D}^2 + \gamma \left\|V_{\theta^*} - V_{\hat{\theta}(T)} \right\|_{{\rm Dir}}^2 \right] \\
 \leq & \frac{1}{2 \sqrt{T} }\left( E\left[ \left\|\bar{\theta}({0}) - \theta^*\right\|_2^2 \right]+\frac{ 2 {\sigma}^2}{N} \right) +\frac{1}{T} \left(\frac{4\hat{R}_0+ 8{\sigma}^2 }{1-\gamma} \right),
\end{align*} which is what we wanted to show.

\textbf{Proof of part (c):} Using that $\gamma \left\|V_{\theta^*} - V_{\bar{\theta}(t)} \right\|_{{\rm Dir}}^2$ is non-negative and rearranging Eq. (\ref{eq:recc_relation}), we have
\begin{small}
    \begin{align*}
 & E \left[ \left\|\bar{\theta}({t+1}) - \theta^*\right\|_2^2 \right]  \leq   \alpha_t^2 \left(\frac{2{\sigma}^2}{N} + \frac{8}{N} \sum_{v \in \mathcal{V}} E\left[ \|V_{\theta_v(t)} - V_{\theta^*} \|_{D}^2\right]\right )
  \\ & +E\left[ \left\|\bar{\theta}({t}) - \theta^*\right\|_2^2 \right]  - 2 \alpha_t (1-\gamma) E \left\|V_{\theta^*} - V_{\bar{\theta}(t)} \right\|_D^2
   .
\end{align*}
\end{small}

Applying Lemma 1 in \cite{bhandari2018finite}, which states that $\sqrt{\omega} \|\theta\|_2 \leq \|V_{\theta}\|_{D} \leq \|\theta\|_2 ,$ we get
\begin{align}
 & E \left[ \left\|\bar{\theta}({t+1}) - \theta^*\right\|_2^2 \right] \leq ( 1-2 \alpha_t (1-\gamma) \omega) E\left[ \left\|\bar{\theta}({t}) - \theta^*\right\|_2^2 \right] \notag \\
 & + \alpha_t^2 \left(\frac{2{\sigma}^2}{N} + \frac{8}{N} \sum_{v \in \mathcal{V}} E\left[ \|V_{\theta_v(t)} - V_{\theta^*} \|_{D}^2\right]\right) \label{eq:1/t}.
\end{align}

We first consider the last term on the right hand side, i.e., $E\left[ \|V_{\theta_v(t)} - V_{\theta^*} \|_{D}^2\right]$. Since each agent in the system executes the classical TD($0$) at time $t$ for $t \in \mathbb{N}_0$, then by part (c) of Theorem 2 and Lemma 1 in \cite{bhandari2018finite}, for $v \in \mathcal{V}$, we have that $E\left[ \|V_{\theta_v(t)} - V_{\theta^*} \|_{D}^2\right] 
 \leq E\left[ \|{\theta_v(t)} - {\theta^*} \|_{2}^2\right] \leq \frac{\hat{\zeta}}{t + \tau},$
where $\hat{\zeta} = \max\left\{ {2 \alpha^2 {\sigma}^2}, \tau \hat{R}_0 \right\}.$  Hence, $\frac{8}{N} \sum_{v \in \mathcal{V}} E\left[ \|V_{\theta_v(t)} - V_{\theta^*} \|_{D}^2\right] \leq \frac{8 \hat{\zeta}}{t + \tau},$ and plugging it into Eq. (\ref{eq:1/t}), we can obtain

\begin{align*}
 & E \left[ \left\|\bar{\theta}({t+1}) - \theta^*\right\|_2^2 \right] \\
 \leq & ( 1-2 \alpha_t (1-\gamma) \omega) E\left[ \left\|\bar{\theta}({t}) - \theta^*\right\|_2^2 \right] + \alpha_t^2 \left(\frac{2{\sigma}^2}{N} + \frac{8 \hat{\zeta}}{t + \tau} \right) \notag\\
 = & \left( 1 - \frac{4}{t+\tau}\right)E\left[ \left\|\bar{\theta}({t}) - \theta^*\right\|_2^2 \right] + \frac{2\alpha ^2 {\sigma}^2/N}{(t+\tau)^2} + \frac{8 \alpha ^2 \hat{\zeta}}{(t+\tau)^3}, 
\end{align*}where we use that $\alpha_t = \frac{\alpha}{t+\tau}$ with $\alpha = \frac{2}{(1-\gamma)\omega}$ and $\tau = \frac{16}{(1-\gamma)^2\omega}$ to get the last line. This recursion now immediately implies part(c) of the theorem using the standard estimate 
\begin{align}
  & \prod_{i=0}^t \left( 1- \frac{4}{t+\tau-i}\right) 
 < \left( \frac{\tau-1}{t+\tau} \right)^4. \label{eq:prod}
\end{align}
\end{proof}

{{ \section{Numerical Experiments}

In this section, we perform some experiments comparing Algorithm \ref{algo:our} with earlier distributed TD methods from \cite{doan2019finite}, \cite{sun2019finite} and \cite{wang2020decentralized} in terms of TD error. Note that the distributed TD methods of \cite{doan2019finite} and \cite{sun2019finite} are the same except that \cite{doan2019finite} has an additional projection step. We consider the case of constant step-size, which is the most widely used in practice, taking  $N = 100$ agents. The communication graph among agents is generated by the Erdos–Renyi model, which is connected. We consider two simple examples: Gridworld (see Chapter 3 of \cite{sutton2018reinforcement}) and  MountainCar-v1 from OpenAI Gym; for the latter, we use the tile coding \cite{sutton2018reinforcement} to discretize continuous state spaces into overlapping tiles. We use 5 tilings, and each tiling has $7 \times 7$ grids. 

{\em Recall that our method only uses one run of average consensus at the end, whereas the other methods require a communication at every step}. The graphs for our method show the TD error at each iteration if we stopped the method and run the average consensus to average the estimates across the network.  Figure \ref{fig:compare_six} shows that the TD errors of Algorithm \ref{algo:our} perform essentially identically to the other methods in spite of the reduced communication. 

\begin{figure}[h]
\centering  
\subfigure[Grid World]{
\label{Fig.sub.1'}
\includegraphics[width=0.2\textwidth]{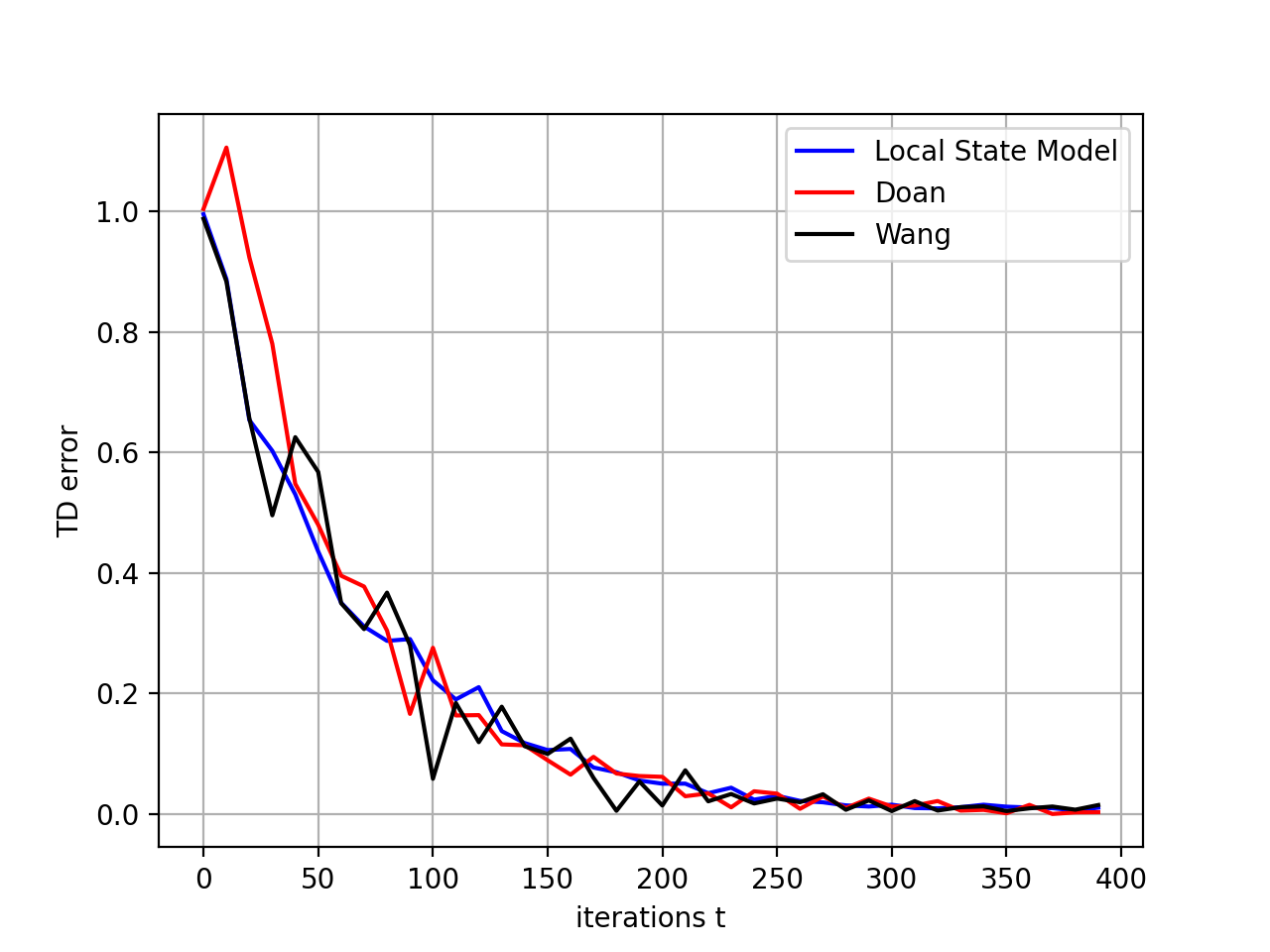}}
\subfigure[MountainCar-v1]{
\label{Fig.sub.4'}
\includegraphics[width=0.2\textwidth]{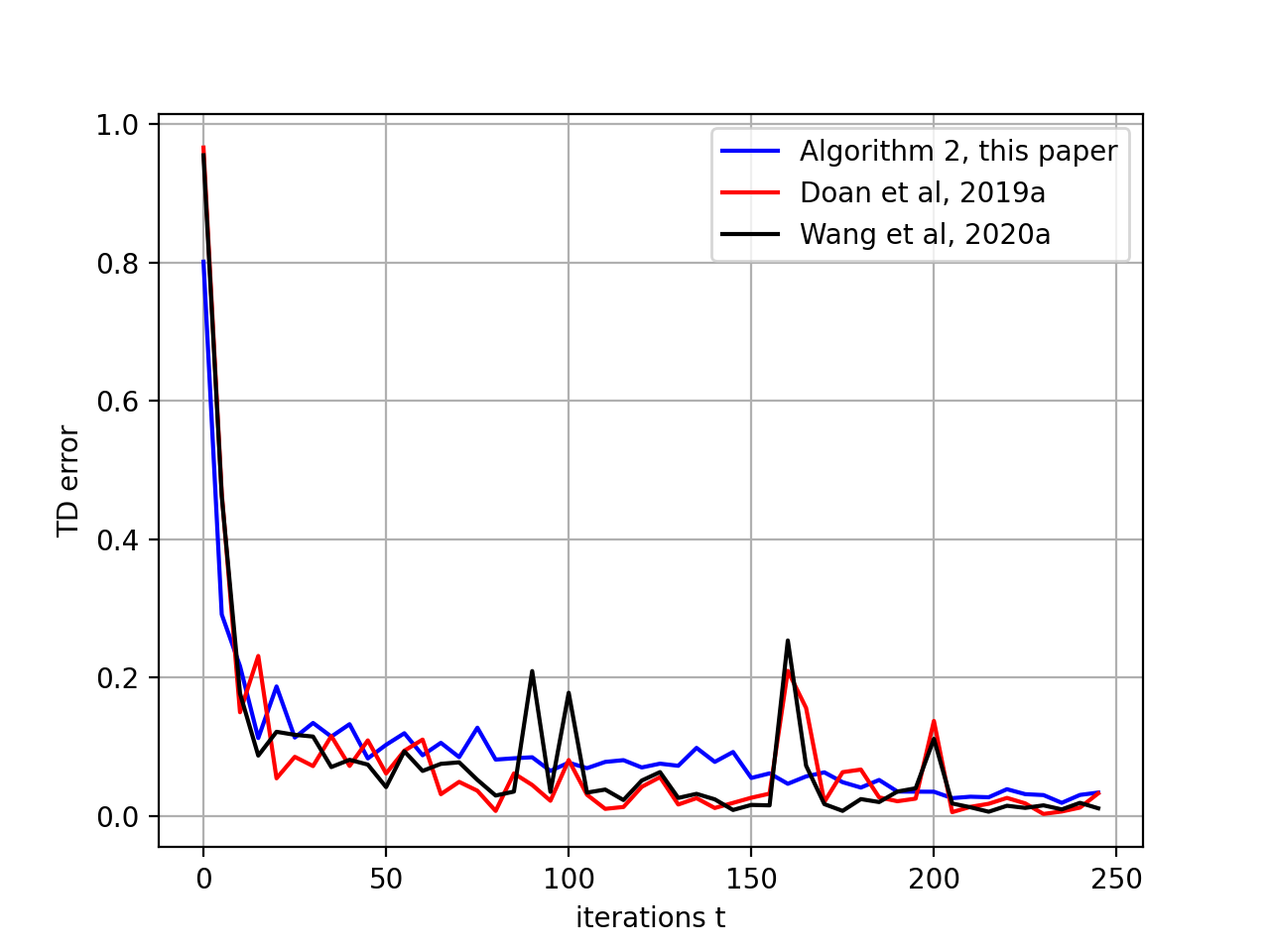}}
\caption{Comparison of our method to the previous literature for a policy that takes uniformly random actions. }
\label{fig:compare_six}
\end{figure}

\section{Conclusion}
We have presented convergence results for  distributed TD($0$) with linear function approximation. Our results are unique in terms of utilizing almost no communication: only one run of average consensus is needed. In particular, this means we need to do $O( \log T)$ average consensus steps for $T$ steps of TD(0) at every node of the network. The convergence bounds we derive reduce the variance by a factor of $N$ when the nodes generate their samples independently. 
The main open question left by this work is whether it is possible to extend these results to other methods popular in the reinforcement learning, such as Q-learning.} It would also be of interest to to apply these results to policies in the context of control of epidemics using the problem formulation in \cite{ma2020optimal}. 

\bibliography{ref}
\bibliographystyle{plain}

\end{document}